\def\year{2018}\relax
\documentclass[letterpaper]{article}
\usepackage{aaai18}
\usepackage{times}
\usepackage{helvet}
\usepackage{courier}
\usepackage{verbatim}
\usepackage{amsmath}
\usepackage{amssymb}
\usepackage{booktabs}
\usepackage{url}  
\usepackage{graphicx}
\usepackage{subfigure}
\usepackage{caption2}
\usepackage{algorithm}
\usepackage{algpseudocode}
\usepackage{amsthm}
\usepackage{algorithmicx} 

\DeclareMathOperator*{\argmax}{arg\,max}

\algnewcommand{\Initialize}[1]{%
  \State \textbf{Initialize:}
  \Statex \hspace*{\algorithmicindent}\parbox[t]{0.9\linewidth}{\raggedright #1}
}

\begin{document}
\title{An Optimal Online Method of Selecting\\ Source Policies for Reinforcement Learning}
\author{\#1719
}
\maketitle
\begin{abstract}
\begin{quote}
Transfer learning significantly accelerates the reinforcement learning process by exploiting relevant knowledge from previous experiences.
The problem of optimally selecting source policies during the learning process
is of great importance yet challenging.
There has been little theoretical analysis of this problem.
In this paper, we
develop an optimal online method to select source policies for reinforcement learning.
This method formulates online source policy selection as a multi-armed bandit problem and augments Q-learning with policy reuse.
We provide theoretical guarantees of the optimal selection process and  convergence to the optimal policy.
In addition, we conduct experiments on a grid-based robot navigation domain to demonstrate its efficiency and robustness by comparing to the state-of-the-art transfer learning method.
\end{quote}
\end{abstract}

\section{Introduction}
Reinforcement learning (RL) \cite{sutton1998reinforcement} is a widely-used framework to learn an optimal control or decision-making policy.
However, RL has a high sample complexity, since a RL agent gains data via repeated interactions with its environment.
Transferring past knowledge to a target task can greatly accelerate reinforcement learning. The first step of transfer in RL is to
select useful knowledge during the reuse process.
If an irrelevant source task is chosen, learning performance could be worse than learning from scratch, which is called negative transfer \cite{pan2010survey}.
This problem is challenging, because in practical situation, the environment is mostly unknown and an agent has no prior knowledge which source task is useful. So the agent has to do online source task selection.

Transfer learning has been recognized as an important direction in RL for a long time \cite{taylor2009transfer}.
Some works leverage source task knowledge without automatically identifying related source tasks \cite{parisotto2015actor,barreto2016successor,gupta2017learning}.
However, these methods may suffer from negative transfer.
Others require humans to define relationships between tasks and relevant source tasks \cite{torrey2005using,taylor2007transfer,ammar2011reinforcement,ammar2015unsupervised}.
Under a more general circumstance, an agent has to do source task selection by itself.
But selecting an appropriate source task usually demands quite a little extra knowledge concerning the domain.
For example, \cite{perkins1999using,lazaric2008transfer,nguyen2012transferring} need some prior experience in the target task. In \cite{ammar2014automated,song2016measuring}, a well-estimated or known model is assumed, which is not always available in practice.
Policy Reuse Q-learning (PRQL) \cite{fernandez2013learning}
requires no prior knowledge of target task or MDP environment,
but it may converge to a suboptimal policy.
To address above limitations, we propose an optimal method to select and reuse source policies online automatically without extra prior knowledge.

Our contributions in this paper can be claimed as follows.
First, we formulate source policy selection problem as a Multi-armed Bandit (MAB) problem where different source policies are regarded as bandits and enable optimal online source policy selection.
Second, we augment Q-learning with policy reuse, while maintaining the same theoretical guarantee of convergence as traditional Q-learning.
Finally, our empirical experiments conducted on a grid-based robot navigation domain verify that our approach (i)
 accomplishes the optimal source policy selection process;
(ii) transfers useful knowledge to a target task and significantly speeds up learning process;
and (iii) achieves virtually empirical performance to traditional Q-learning in equivalent situations where no source knowledge is useful.

In the remainder of this paper, we start by reviewing related work. Then, background knowledge on RL and problem formulation is described. After that, we present our approach and theoretical results followed by empirical results comparing our approach with the state-of-the-art algorithm. Finally, we conclude and outline directions for future work.
\section{Related Work}
As transfer in RL has received much attention recently, we now discuss in greater detail the relationship between our algorithm and others.
\cite{talvitie2007experts} treated previously learned policies as experts and mediated these experts intelligently. Since the mixing time of experts is known in episodic domains, this method is not as effective as standard algorithms.
In contrast, our approach works fine for episodic tasks.
PRQL \cite{fernandez2013learning}
selected source tasks from a library probabilistically using a softmax method.
However, because it stops doing exploration soon after greedy policy's reward exceeds the reuse reward,
it does not guarantee convergence to the optimal policy.
\cite{sinapov2015learning} learned transferability between a source-target task pair using meta-data.
This method has a low efficiency, for it is expensive to generate a large set of data using every source-target pair.
On the contrary, our approach adapts an MAB method to select source policies and accomplishes the optimal selection process.
\cite{rosman2016bayesian} proposed a Bayesian method to do policy reuse in a policy library. But
it mainly solves the problem of short-lived sequential policy selection,
therefore this method does not learn a full policy.
\cite{rusu2016progressive} leveraged prior knowledge via lateral connections to previously learned features in neural networks. Although they have showed positive transfer even in orthogonal or adversarial tasks, there is no theoretical foundation of their algorithm.

Some related works focus on multi-task learning (MTL), which is very similar to transfer learning.
 MTL assumes that all MDPs are drawn from the same distribution and learning is parallel on several tasks \cite{ramakrishnan2017perturbation}.
 In contrast, we make no assumption regarding the distribution over MDPs and concentrate on transfer learning problem.
 In one previous MTL work, \cite{wilson2007multi} represented the distribution of MDPs with a hierarchical Bayesian model. The continuously updated distribution served as a prior for rapid learning in new environments. But as Wilson et al mentioned in their work, their algorithm is not computationally efficient.
In more recent MTL works,
\cite{brunskill2013sample} proposed a technique that involves two phases of learning to reduce sample complexity of RL.
\cite{fachantidis2015transfer} determined the most similar source tasks based on compliance, which can be interpreted as a sort of distance metric between tasks.

\section{Preliminaries and Problem Formulation}
This section briefly reviews RL background and describes problem formulation. RL is a dominant framework to solve control and decision-making problems via mapping situations to actions. The learning environment of RL is an MDP defined as a tuple of $<S,A,T,R,b_0,\gamma>$, where $S$ denotes a discrete state space. At a time step $t$, an agent in a specific state $s_t$ performs an action $a_t$ in a discrete action space $A$.
Based on transition function $T(s_t,a_t,s_{t+1})$, the agent switches to next state $s_{t+1}$ and gets a reward $r_t$ according to reward function $R(s_t,a_t,s_{t+1})$.
 An agent begins to interact with its environment from a start state sampled from an initial belief $b_0$ and
 keeps taking actions until it reaches a final state or an absorbing state. A policy $\pi$ directs the agent which action to take, given a particular state. The agent's goal is to learn an optimal policy that maximizes the expected value of cumulative reward after training.
$\gamma$ is a discount factor to reduce the impact of future rewards on learning policy.
Q-learning is a model-free RL method, which is able to find an optimal policy for any finite MDP.
 A Q-learning agent learns the expected utility for each action in every given state $Q(s,a)$ by doing value iteration update in each step as below:
\begin{align*}
Q(s_t,a_t)\leftarrow(1-\alpha)Q(s_t,a_t)+\alpha(r_t+\gamma\max\limits_{a}Q(s_{t+1},a))
\end{align*}
where $\alpha$ is a learning rate.

Given a source policy library $L=\{\pi_1,...,\pi_n\}$, where $\forall \pi_j\in L$ denotes the optimal policy for source task $\Omega_j$ in one domain, our goal is to reuse source policies in library $L$ optimally and solve a target task $\Omega$. We assume that tasks are episodic with maximum number of steps $H$ in each episode and take average reward of $k$ episodes $W(k)$ as a metric of a learning algorithm:
\begin{align*}
W(k)=\frac{1}{k}\sum\nolimits_{i=0}^k\sum\nolimits_{h=0}^H\gamma^hr_{i,h}
\end{align*}
where $r_{i,h}$ is the reward of time step $h$ in episode $i$.
The convergence speed and value of $W(k)$ indicate the learning performance.
Our approach applies transfer learning to Q-learning,
so it is an off-policy learning method.
Since the exploration strategy will affect average reward during learning greatly, we do evaluation following a fully greedy strategy after each learning episode and obtain a learning curve of average reward.
\section{Approach}
The rewards of reusing source policies are stochastic in RL. Therefore, there is a dilemma between exploiting the policy which yields high current rewards and exploring the policy which may produce more future rewards.
We adapt an MAB method for this problem.
Our approach accomplishes online source policy selection via evaluating the utility of each source policy during learning a target task.
The exploration process is guided by the intelligently selected policy in Q-learning, which is an off-policy learner.
In the situation where no source knowledge is useful, $\epsilon\text{-}greedy$ strategy in our approach will play a major role to maintain the same learning performance as traditional Q-learning.
In this section, we firstly present an optimal online source policy selection method.
After that, we introduces how to reuse source knowledge through Q-learning.
Finally, we provide theoretical optimality analysis for our algorithm.
\begin{figure}
   \centering
    \includegraphics[width = .47\textwidth]{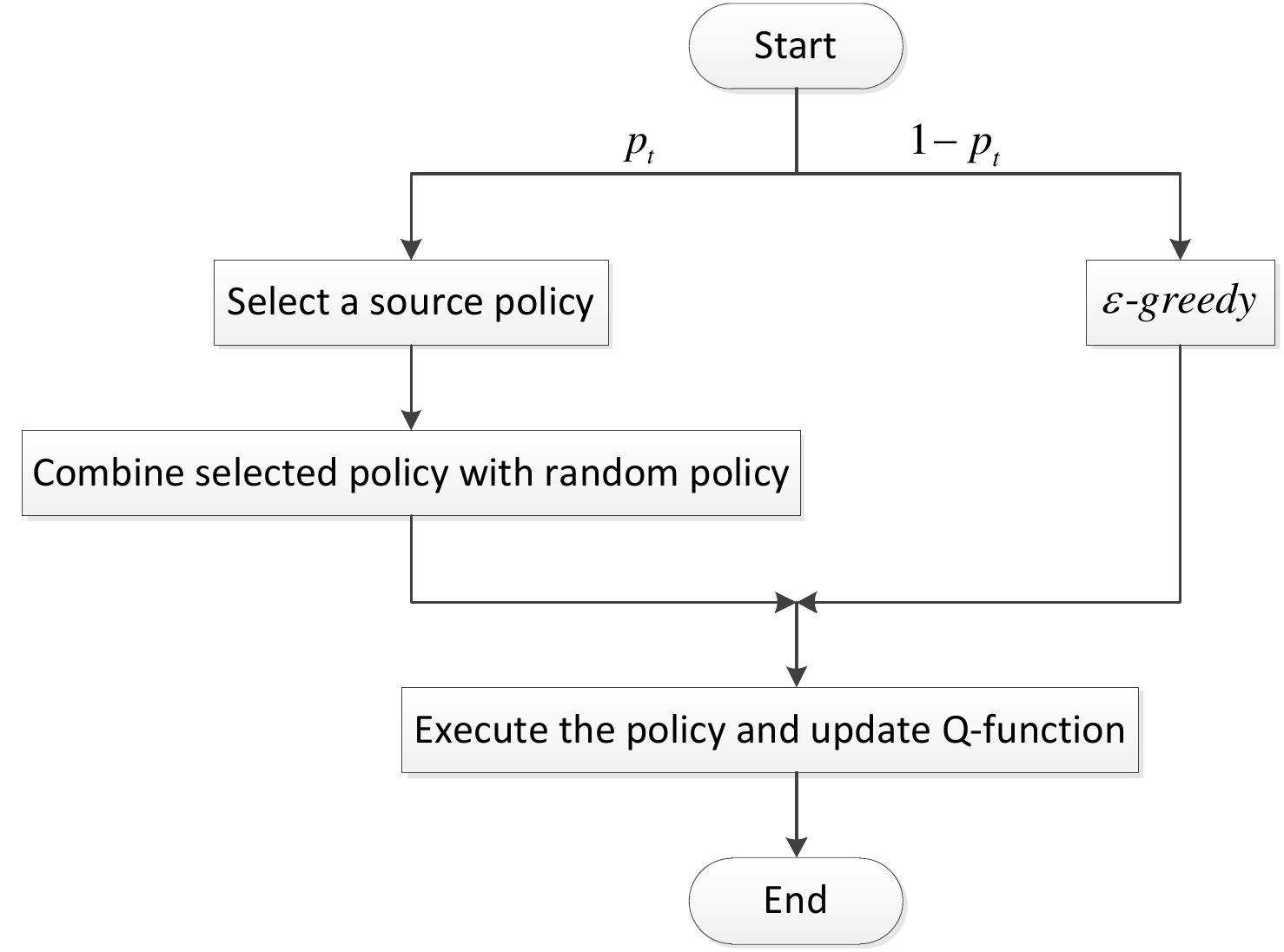}
   \setlength{\abovecaptionskip}{-0.1cm}
   \setlength{\belowcaptionskip}{-0.5cm} 
   \caption{Flow chart of our algorithm }
\end{figure}

Figure 1
provides the overview of our algorithm.
Firstly, we execute $\epsilon\text{-}greedy$ exploration strategy with a probability of $1-p_t$, and select a source policy from a library using an MAB method with a probability of $p_t$, where $p_t$ decreases over time.
That is to say, at the beginning of learning, we exploit source knowledge mostly. As learning goes, past knowledge becomes less useful, so we exploit $\epsilon\text{-}greedy$  strategy to go on learning.
 In order to exploit a past policy,
 our algorithm combines the past policy with the random policy. Once the policy of current episode is determined, our algorithm will execute the policy and update Q-function.
 Therefore, how to optimally select a source policy and to reuse the selected policy are key to transfer learning.
\subsection{Source Policy Selection}
An agent has no prior knowledge which source policy is useful for current target task before trying.
  It has to decide which source policy to reuse in the next episode based on previous rewards to obtain a larger cumulative reward.
  Different source policies can be regarded as bandits with stochastic rewards in MAB.
   Source policy selection and MAB are both the exploitation versus exploration tradeoff in essence.

The pseudo code of our source policy selection ($\pi$-selection) method is shown in Algorithm 1. In order to solve target task $\Omega$, this method chooses $\epsilon\text{-}greedy$ policy with a probability of $1-p_t$  and chooses a policy from source policy library $L=\{\pi_1,...,\pi_n\}$ using UCB1 with a probability of $p_t$ (Line:3-5).
When a past policy is chosen, we execute combining method of Algorithm 2 on the chosen past policy $\pi_k$ in episode $k$ (Line:6-7). We need to keep policy $\pi_j$ average gain $W_j(k)$ and number of selected times $T_j(k)$ in the previous $k$ episodes ($\forall j=1,...,n$) for UCB1 (Line:11-16). $\psi_h$  controls the reuse degree and Algorithm 1 is executed for $K$ episodes.

UCB is a simple and efficient algorithm that achieves optimal logarithmic regret of MAB \cite{lai1985asymptotically}.
MAB defines a collection of $n$ independent bandits with different expected reward $\{\mu_1,\mu_2,...,\mu_n\}$.
An agent sequentially selects bandits to make cumulative regret minimum. The regret is difference between expected reward of the selected bandit and maximum reward expectation $\mu^*$.
UCB1 of UCB family algorithms maintains number of steps $T_i(k)$ where machine $i$ has been selected in $k$ steps and empirical expected reward $\hat{\mu}_i$. Each machine in the collection is played once initially and UCB1 selects machine $j(k)$ as follows in every time step afterwards:
\begin{align*}
j(k)=\arg \max\limits_{i=1...n}(\hat{\mu}_i+\sqrt{\frac{2\text{ln}(k)}{T_i(k)}})
\end{align*}
\begin{algorithm}[!hbt]
\caption{$\pi$-selection ($\Omega,L,p_t,\psi_h,K,H$).}
\label{alg:Framwork}
\begin{algorithmic}[1]
\Initialize{
$Q(s,a)=0,\forall s\in S,a\in A $\\
$\forall j=1,...,n:$
$[W_j(0),Q]\leftarrow\pi\text{-reuse}(\pi_j,\Omega,Q,\psi_h,H)$\\
$T_j(0)=1$
}
\For{$k=1$ to $K$}
        \State {Choose a policy $\pi_k$:}
        \State {\quad With a probability of $1-p_t$, $\pi_k=\pi_{\epsilon\text{-}greedy}$}
        \State {\quad With a probability of $p_t$:}
        \begin{displaymath}
        \setlength\abovedisplayskip{1pt plus 3pt minus 7pt}
        \setlength{\belowdisplayskip}{0.1pt}
        \small
\hspace{-5mm} j=\argmax_{1\leq j\leq n}(W_j(k-1)+\sqrt{\frac{2\text{ln}(\sum\nolimits_{j=1}^nT_j(k-1))}{T_j(k-1)}})
\end{displaymath}
 \begin{displaymath}
        \setlength\abovedisplayskip{1pt plus 3pt minus 7pt}
        \setlength{\belowdisplayskip}{0.1pt}
        \small
\hspace{-5mm} \pi_k=\pi_j
\end{displaymath}
         \If{$\pi_k \in L$}
              \State $[R,Q]\leftarrow\pi$-reuse$(\pi_k,\Omega,Q,\psi_h,H)$
            \Else
              \State Execute $\epsilon\text{-}greedy$ policy
            \EndIf
        \If{$\pi_k \in L$}
              \State {$W_j(k)=\frac{W_j(k-1)T_j(k-1)+R}{T_j(k-1)+1}$,}
              \State {$T_j(k)=T_j(k-1)+1$;}
            \Else
              \State $W_j(k)=W_j(k-1), T_j(k)=T_j(k-1)$;
            \EndIf
      \EndFor
\label{code:fram:select} \\
\Return $Q(s,a)$
\end{algorithmic}
\end{algorithm}
\subsection{Source Policy Reuse}
To take full advantage of the selected policy, the random policy is indispensable for interacting with unexplored states. Without random actions, past policies will lead an agent to their original goals rather than the goal of target task. Exploiting the useful source policy can be regarded as directed exploration. Therefore, we combine source policy $\pi_{past}$ with random policy $\pi_r$ probabilistically in policy reuse strategy demonstrated by Algorithm 2.
At time step $h$, we take action based on $\pi_{past}$ with a probability $\psi_h$, and take a random action with a probability of $1-\psi_h$ (Line:3-4).
\begin{algorithm}[htbp]
\caption{$\pi\text{-reuse}$ ($\pi_{past},\Omega,Q,\psi_h,H$).}
\label{alg:Framwork}
\begin{algorithmic}[1]
        \State {Set initial state $s$ randomly.}
        \For{$h=1$ to $H$}
        \State {With a probability of $\psi_h, a=\pi_{past}(s)$}
        \State {With a probability of $1-\psi_h, a=\pi_r(s)$}
        \State {Receive next state $s'$ and reward $r_{h}$}
        \State {Update $Q(s,a)$:}
        \begin{equation*}
        \setlength\abovedisplayskip{1pt plus 3pt minus 7pt}
        \setlength{\belowdisplayskip}{0.1pt}
        \small
        \hspace{-5mm} Q(s,a)\leftarrow(1-\alpha)Q(s,a)+\alpha(r_{h}+\gamma\max\limits_{a'}Q(s',a'))
        \end{equation*}
        \State {Set $s \leftarrow s'$}
      \EndFor
      \State {$R=
      \sum\nolimits_{h=0}^H\gamma^hr_{h}$}
\label{code:fram:select} \\
\Return $R$ and $Q(s,a)$
\end{algorithmic}
\end{algorithm}

Algorithm 2 shares some similar ideas with PRQL.
They both take an action probabilistically in one episode. However,
we  mix no greedy action in reuse episodes to maintain a fixed expected value of reuse reward. In addition, each source policy is uncorrelated so stochastic assumption of MAB is satisfied.
 We choose $\epsilon\text{-}greedy$ strategy with a probability of $1-p_t$ outside reuse episodes instead. $\epsilon\text{-}greedy$ strategy is very crucial when there is no useful source policy in the library.
 As $p_t$ decreases over time, our algorithm will reuse source policy less and converge to $\epsilon\text{-}greedy$ strategy.
\subsection{Theoretical Analysis}
We present two theoretical results that provide the foundation of our approach below.
\newtheorem*{theorem}{Theorem 1}
\begin{theorem}
Given a source policy library $L$, if UCB1 selects  source policy according to the reward of each episode in Algorithm 2,
the expected regret is logarithmically bounded.
\end{theorem}
\begin{proof}
Because there is no greedy action in each reuse episode, all the source policies are not correlated.
In addition,
for each policy $\pi_j$, its reward in every episode is an independent sample from the same distribution with a fixed expectation.
So the stochastic assumption of MAB is satisfied. UCB can achieve the logarithmic regret bound asymptotically which is proved minimum by Lai and Robbins in their classical paper \cite{auer2002finite}, so it is an optimal allocation strategy when there is no preliminary knowledge about the reward distribution. As a result, this method of selecting source policies from a library is theoretically optimal.
\end{proof}
\newtheorem*{theorem1}{Theorem 2}
\begin{theorem1}
$Q(s,a)$ of Algorithm 1 will converge to the optimal Q-function $Q^*$ for target task $\Omega$ as traditional Q-learning.
\end{theorem1}
\begin{proof}
Since $p_t$ controlling the exploration rate decreases with time, Algorithm 1 will execute $\epsilon\text{-}greedy$ policy more and more frequently.
$\epsilon$ of $\epsilon\text{-}greedy$ policy is less than 1, so Algorithm 1 will keep doing exploration for infinite episodes.
Q-learning with a proper learning rate converges to the optimal Q-function
for any finite MDP \cite{melo2001convergence}.
The probability of executing random actions in Algorithm 1 will never equal to 0, so all state-action pairs will be visited infinitely often.
As a result, Algorithm 1 will converge to the optimal Q-function as traditional Q-learning.
\end{proof}
Both our learning method and selection method are optimal. Thus, our approach is an optimal online strategy to select source policies theoretically.
\section{Empirical Results}
To demonstrate that our algorithm is empirically sound and robust, we carry out experiments in a grid-based robot navigation domain with multiple rooms and compare the results with the state-of-the-art algorithm, PRQL \cite{fernandez2013learning}.
\subsection{Experimental Settings}
Our navigation domain has been used by PRQL paper.
Some recent works of transfer learning also conduct experiments in a similar grid-world domain \cite{lehnert2017advantages,laroche2017transfer}.
The map of our navigation domain is composed of $N\times M$ ($21\times 24$ in our case) states which denote free positions, goal area and wall. Each state is plotted as a $1\times 1$ cell. An agent's position is represented by a two-dimensional coordinates $(x,y)$ continuously. Afterwards, we take integer part of $(x,y)$ to determine discrete state of an agent. The agent in this problem can take four actions, respectively up, down, left and right. Arbitrary action makes agent's position move in the corresponding direction with a step size of 1. To make this problem more practical, we design a stochastic MDP by adding a uniformly distributed random variant within $(-0.2,0.2)$ to $x$ and $y$ respectively after an action. When an agent reaches a wall state, the wall will keep the agent in the state before taking actions. After an agent reaches the goal area, it will obtain a reward of 1 and then this episode ends.
Arriving at the other states except the goal state generates no reward.
An agent has no high-level view of the map and only observes its current state.
\vspace{-8pt}
\begin{figure}[htbp]
   \centering
   \includegraphics[width = .47\textwidth]{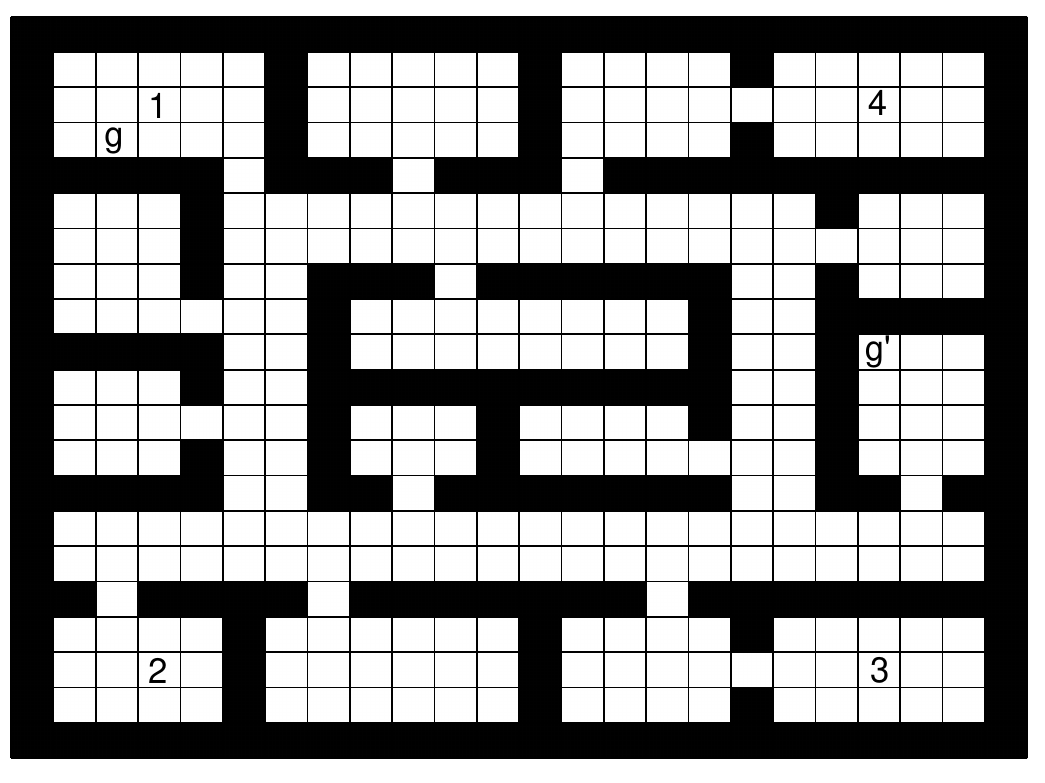}
   \setlength{\abovecaptionskip}{-0.1cm}
   \setlength{\belowcaptionskip}{-0.2cm} 
   \caption{ Target tasks and source task library in the map }
\end{figure}

In Figure 2, $g$ and $g'$ represent the goals of target task $\Omega$ and $\Omega'$.
Numbers in Figure 2 denote goals of a source task library
 $\{\Omega_1,\Omega_2,\Omega_3,\Omega_4\}$.
Source policies in the library are optimal to their respective tasks. Obviously, $\Omega_1$ is the most similar to $\Omega$, because their goals are in the same room, and the other tasks are dissimilar to $\Omega$.
This problem has a large number of states, and the initial belief $b_0$ is an uniform distribution. Therefore, learning from scratch is rather slow and transfer learning can significantly accelerate the learning process.

In this experiment,
$\alpha=0.05,\gamma=0.95$ for Q-function update.
$K$ is set as 4000, which is enough for our approach to learn a policy with high reward.
To avoid an agent entering a dead loop, $H$ is set as 100.
Because goals of different tasks are different,an agent takes actions according to the past policy with a larger probability at the beginning of an episode.
Afterwards, it takes more random actions.
So $\psi_h$ is set as $0.95^h$.
In addition, $p_t$ is set to $1-(k/k+1500)$, which decreases over time.
So our approach converges to $\epsilon\text{-}greedy$ at last. To make parameters simple,
 $\epsilon$ of $\epsilon\text{-}greedy$ policy is set as 0.1 (with probability 0.9 follows the greedy strategy, and with probability 0.1 acts randomly).
We conduct these experiments with PRQL and Q-learning to do comparison.
Q-learning utilizes $\epsilon\text{-}greedy$ with the same parameter as $\epsilon\text{-}greedy$ of our approach.
The parameters of PRQL are consistent with those in \cite{fernandez2013learning}.

UCB1-tuned tunes upper confidence bound according to variance of bandit rewards as follows:
\begin{align*}
j(k)=\arg \max\limits_{i=1...n}(\hat{\mu}_i+\sqrt{\frac{c\text{ln}(k)}{T_i(k)}})
\end{align*}
where
\begin{align*}
c=\text{min}(\frac{1}{4},\hat{\sigma}_i^2(k)+\sqrt{\frac{2\text{ln}(k)}{T_i(k)}})
\end{align*}
($\frac{1}{4}$ is upper bound on the variance of a Bernoulli random variable). UCB1-tuned outperforms UCB1 in a multitude of  experiments. Although it has not been proved theoretically optimal, another algorithm UCB-V which also considers the variance of bandit rewards has already been proved optimal in theory \cite{audibert2009exploration}. As the variance and expectation of reward in this experiment are much smaller than $\frac{1}{4}$, we set $c$ to 0.0049 for Algorithm 1. A lager $c$ will lead to a higher exploration rate, so it is more suitable to the circumstance with a larger variance.

In next section, we compare the empirical performance between our algorithm and PRQL especially.
\subsection{Experimental Results}
 In order to manifest that our approach achieves the optimal source policy selection, we firstly show the learning curve by doing evaluation after each episode and the frequency of selecting source policies.
Next, we compare the expected reward among PRQL, our approach and traditional Q-learning.
Then, we set reward functions of target tasks randomly and conduct the above experiment again to demonstrate robustness of our approach.
Afterwards, we conduct an experiment to indicate that our method is equally applicable to a circumstance where no similar task exists in source task library.
Finally, we present a scenario when PRQL does not converge to the greedy policy. However, in the same case, our approach still converges to $\epsilon\text{-}greedy$ strategy.
\vspace{-8pt}
\begin{figure}[htb]
   \centering
   \includegraphics[width =0.47\textwidth]{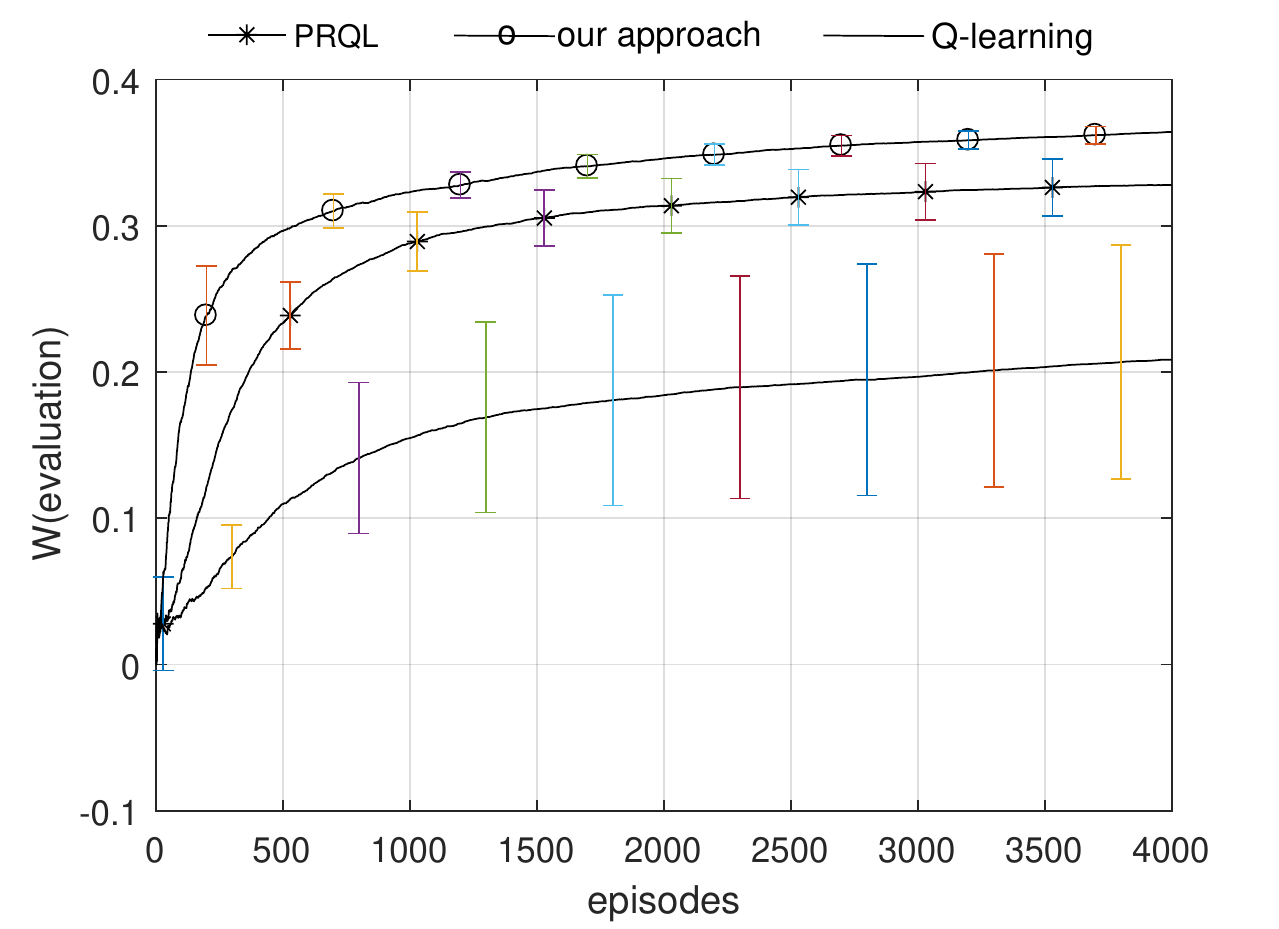}
   \setlength{\abovecaptionskip}{-0.3cm}
  \setlength{\belowcaptionskip}{-0.3cm}
   \caption{Performance comparison of PRQL, our algorithm and traditional Q-learning on target task $\Omega$}
\end{figure}

Figure 3 shows the learning curve of PRQL, our approach and traditional Q-learning when solving target task $\Omega$.
We compare the average reward generated by following a fully greedy policy after each episode, since these three learning methods are all off-policy.
Each learning process in Figure 3 has been executed 10 times. The average value is shown and error bars represent standard deviations.

In Figure 3, $x$ axis  represents the number of episodes and $y$ axis represents average evaluation reward.
From Figure 3, we can have three observations.
First, our algorithm uses less time to reach a threshold of average reward than PRQL. Average reward of our approach is greater than 0.3 in only 500 episodes and more than 0.35 in the end. Second, the asymptotic performance of our approach is
 better than PRQL. Although the gap is not overwhelmingly big, it is comparable to the convergence value of cumulative reward, since reward in this experiment is sparse and not exceedingly large. Third, reward of Q-learning increases much more slowly than our approach, so knowledge transfer in our approach is intensely efficient and no negative transfer occurs. In addition, standard deviations of our approach are the smallest among the three curves, which demonstrates that our approach has an extremely stable performance in the 10 executions.
 \vspace{-8pt}
 \begin{figure}[htbp]
   \centering
   \subfigure[our approach]{
   \centering
   \includegraphics[width =0.46\textwidth]{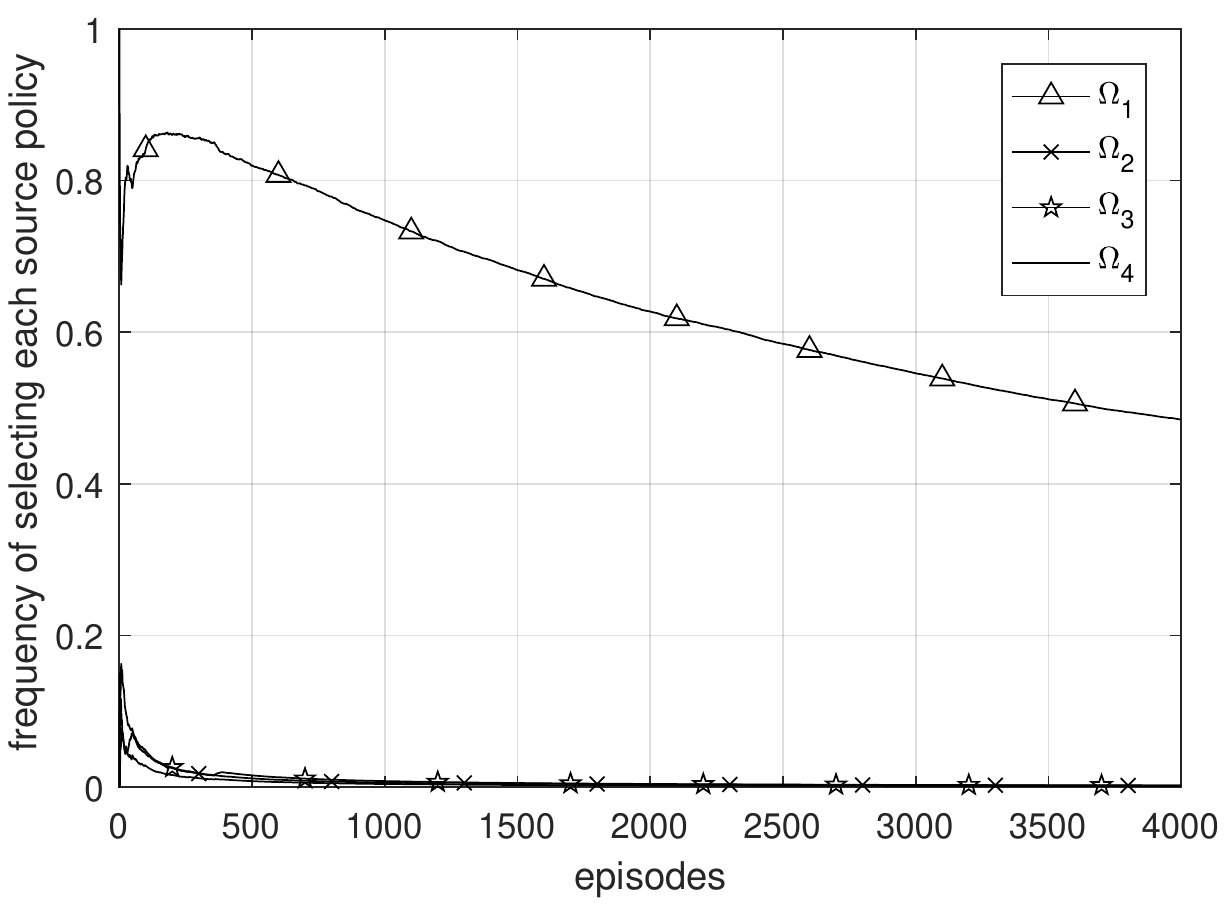}
   }

   \subfigure[PRQL]{
   \centering
   \includegraphics[width =0.46\textwidth]{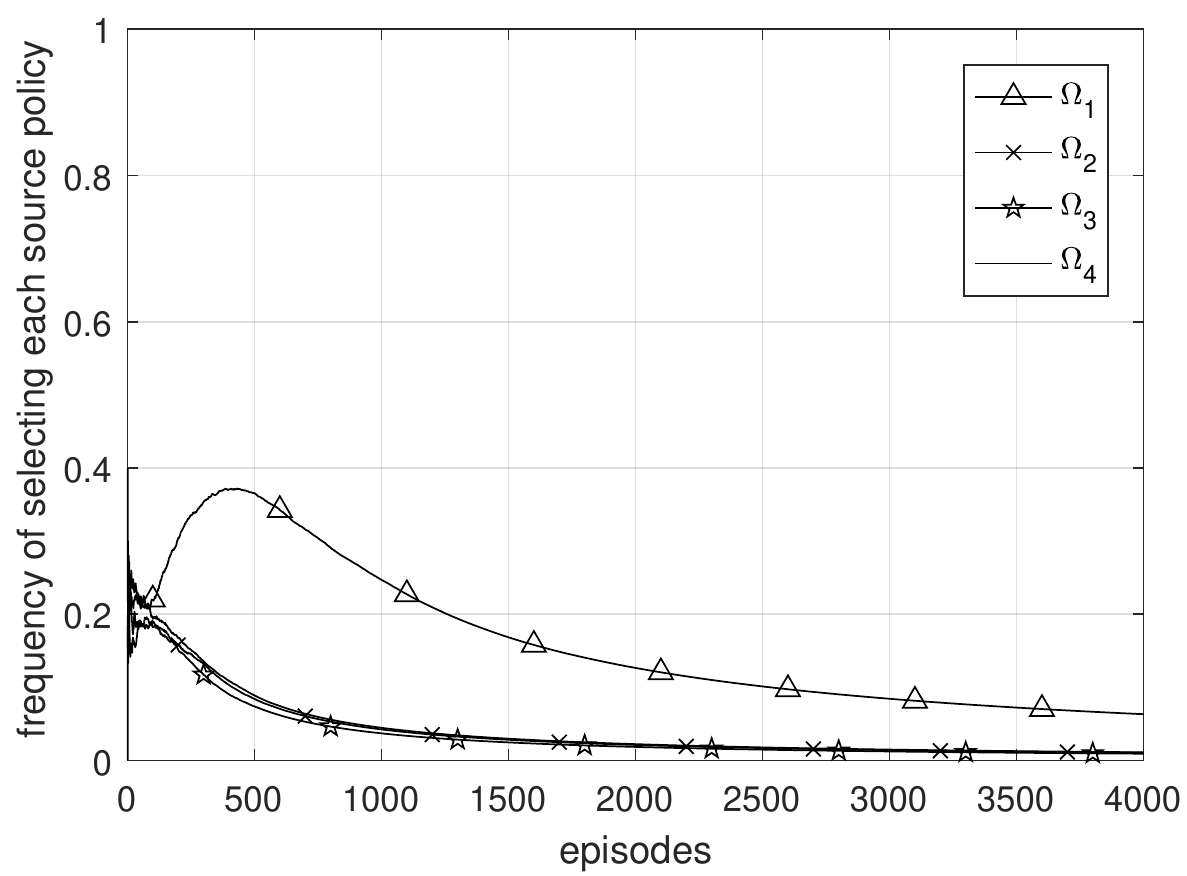}
   }
   \setlength{\abovecaptionskip}{-0.3cm}
   \setlength{\belowcaptionskip}{-0.3cm}
   \caption{Frequency of selecting each source policy}
\end{figure}

Since our approach selects source policies deterministically, we cannot compare the probability of selecting each source policy.
Therefore, we show the frequency of selecting source policies by our approach and PRQL in Figure 4.

We can see that our approach almost does not select irrelevant tasks any more after 500 episodes. In contrast, frequency of choosing dissimilar tasks by PRQL is still around 0.1 in 500 episodes. It costs a shorter time for our method to detect that source task $\Omega_1$ is proper to transfer from.
The curve of $\Omega_1$ drops slowly in our method because we keep using the past policy to explore the environment and our reuse method is different from PRQL. To satisfy the stochastic assumption of MAB, we split greedy actions from the reuse process.
As frequency is different from probability,
we select $\epsilon\text{-}greedy$ policy with a high probability when $k$ is large.

Since the initial belief $b_0$ in this experiment is an  uniform distribution, the expected reward can be denoted as $\frac{1}{S}\sum\nolimits_{s\in S}\max\nolimits_{a}Q(s,a)$.
Figure 5 shows the expected reward of our approach, PRQL and Q-learning.
\begin{figure}[htbp]
   \centering
   \includegraphics[width =0.47\textwidth]{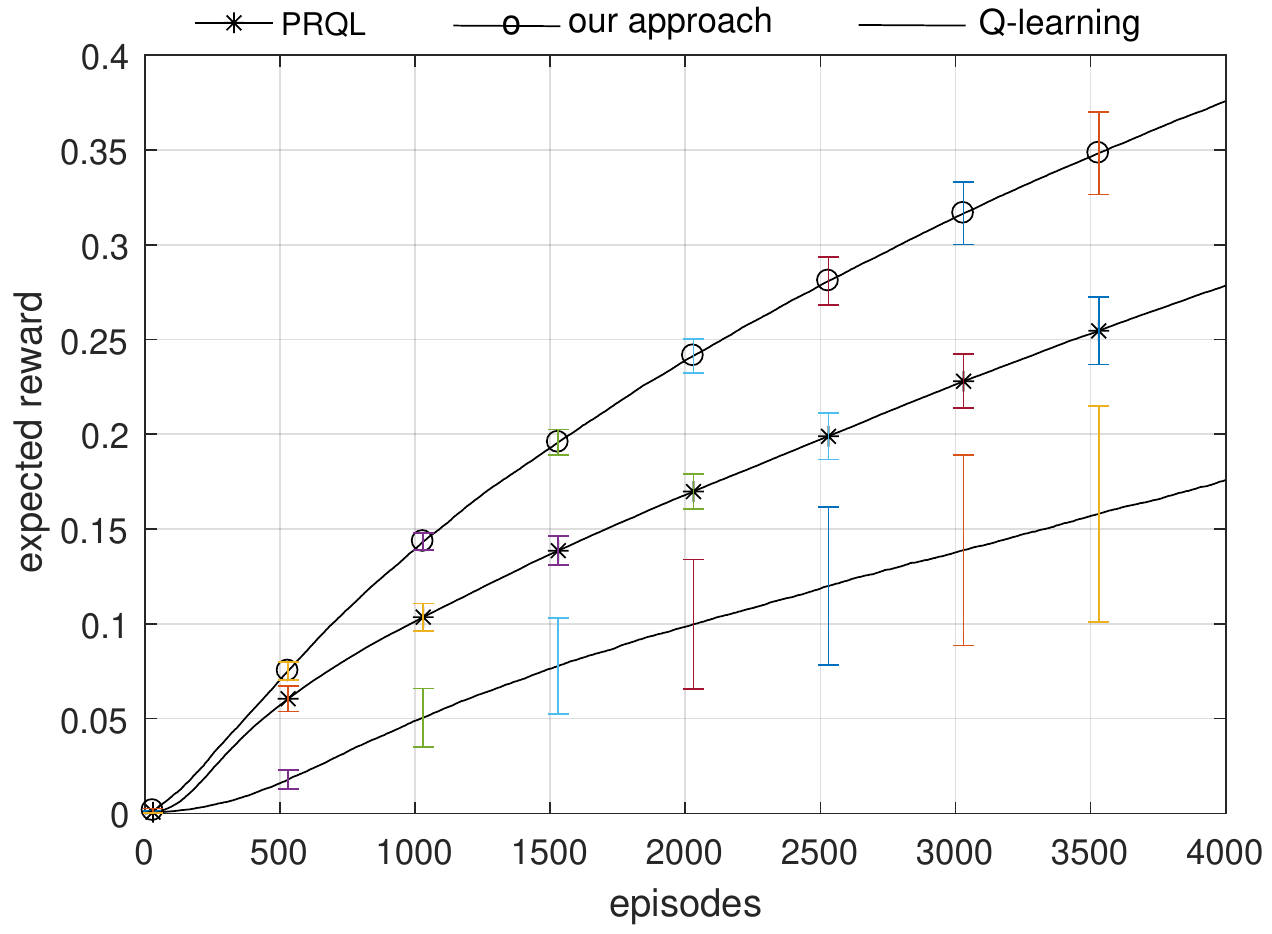}
   \setlength{\abovecaptionskip}{-0.3cm}
   \setlength{\belowcaptionskip}{-0.3cm}
   \caption{Expected reward of  PRQL, our approach and Q-learning}
\end{figure}

In Figure 5, the curve of our approach rises fastest, which demonstrates that our agent reaches the goal more times using the same time. So we obtain more rewards during learning and these rewards ``back up" to other states.
The expected reward converges slowly,
for only reaching the goal state generates a reward. Therefore, we have not shown the convergent part, so the curves seem polynomial.

To demonstrate robustness of our algorithm, we randomly select goals of target tasks, guaranteed that these goals are in the same room as one of source tasks in Figure 2.
So there is a similar source task to transfer from.
In next experiment, we discuss the situation where no source knowledge is useful. We choose 9 different goals for this experiment, which are shown in Figure 6 with numbers.
\begin{figure}[htb]
   \centering
   \includegraphics[width =0.47\textwidth]{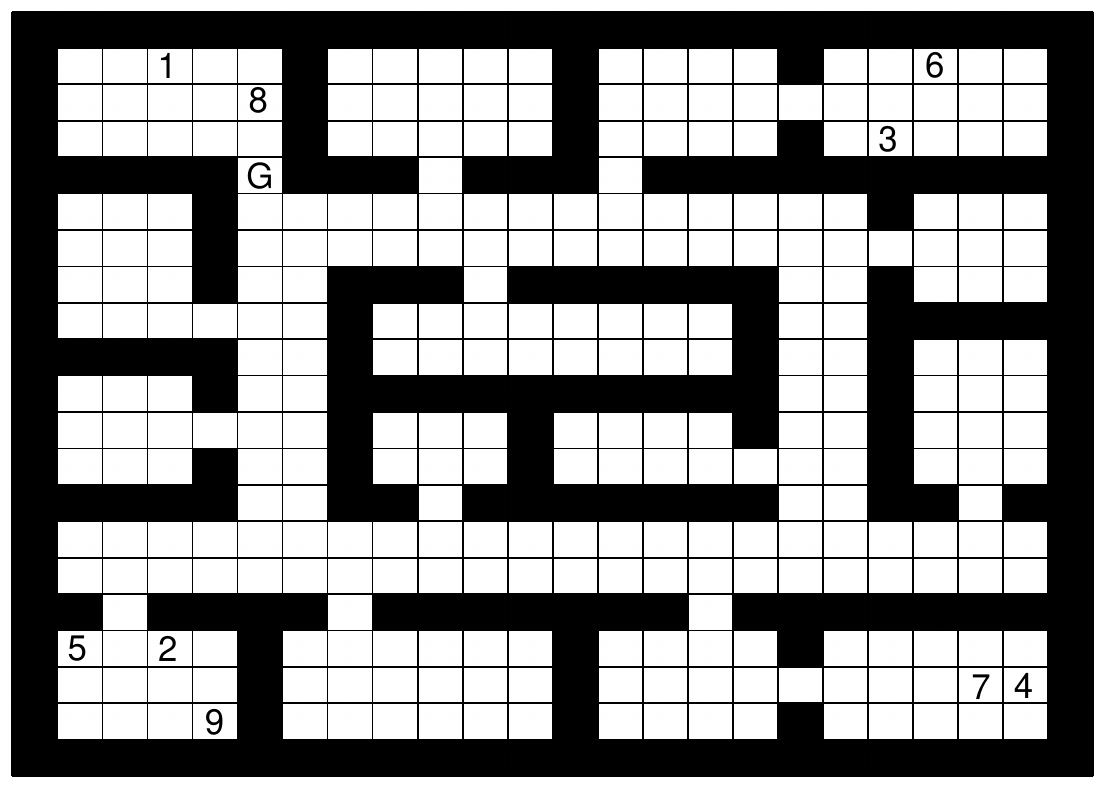}
   \setlength{\abovecaptionskip}{-0.3cm}
   \setlength{\belowcaptionskip}{-0.5cm}
   \caption{Randomly selected target tasks}
\end{figure}
\begin{figure*}[htb]
   \centering
   \includegraphics[width =0.99\textwidth]{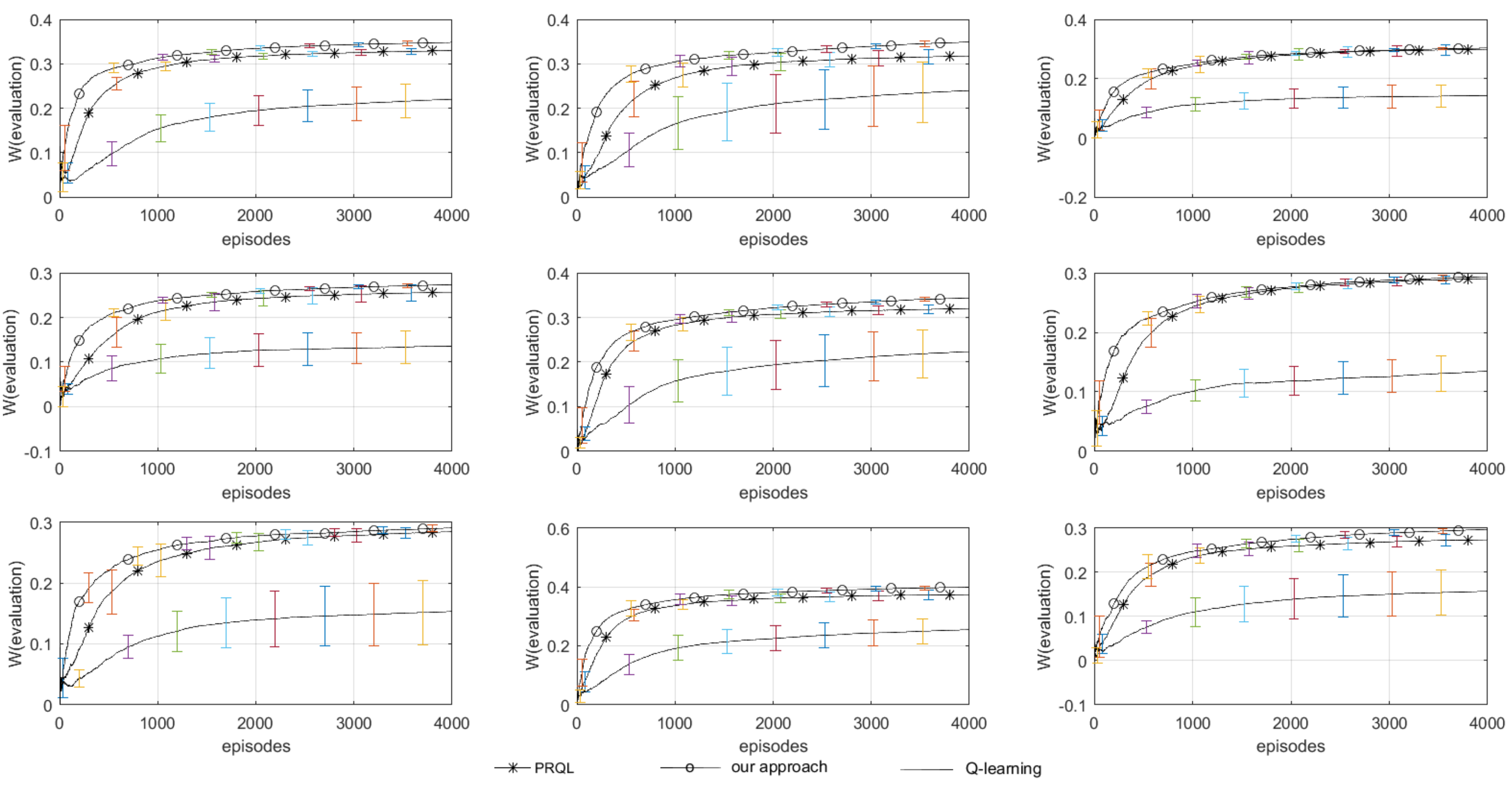}
   \setlength{\abovecaptionskip}{-0.1cm}
   \setlength{\belowcaptionskip}{-0.3cm}
   \caption{Average reward of our algorithm, PRQL and Q-learning to solve randomly selected target tasks}
\end{figure*}

 Figure 7 shows average evaluation reward of our algorithm, PRQL and Q-learning when solving tasks denoted in Figure 6.
The reward of our algorithm has a faster increment than PRQL in all the cases. Moreover, convergence value of reward in our algorithm is sometimes a little larger.
PRQL converges to the greedy policy and does not explore any more, soon after the reward of greedy policy exceeds all source policies.
Thus, it may converge to a suboptimal policy in the end. However, our approach keeps doing exploration, so every state-action pair is visited infinitely often. As a result, our algorithm surely converges to the optimal policy.

To indicate that our method can be applied to a circumstance where  there is no similar task in source task library, we conduct the above experiment again to solve target task $\Omega'$ in Figure 2 using the same source task library.
As we can see, the goal of $\Omega'$ is in a totally different room compared to the goals of source tasks.
So all the source policies in the library are useless for $\Omega'$.
The performance comparison of PRQL, our approach and Q-learning is shown in Figure 8.

\vspace{-8pt}
\begin{figure}[htb]
   \centering
   \includegraphics[width =0.47\textwidth]{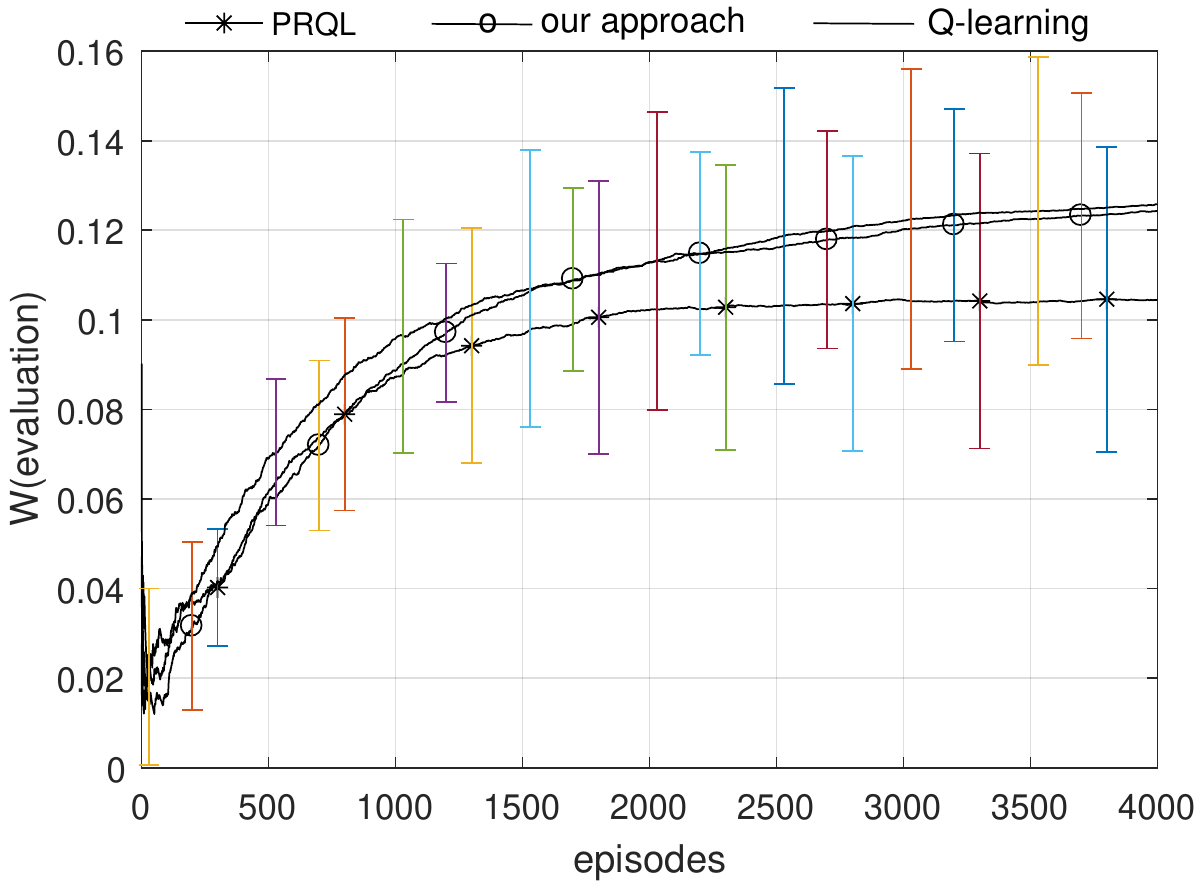}
   \caption{Performance comparison of PRQL, our algorithm
and traditional Q-learning on target task $\Omega'$}
\end{figure}

\vspace{3mm}
 As shown in Figure 8, during the first 1000 episodes, the three curves have a similar growth trend, for all the algorithms do exploration in their own way at the beginning. Afterwards, the curve of PRQL starts to flatten, because it has converged to the greedy policy and no random actions is taken. In contrast,
 our approach has almost the same performance as Q-learning and their curves keep rising, since these two methods go on doing exploration with $\epsilon\text{-}greedy$ strategy.
$\epsilon\text{-}greedy$ strategy of our approach plays a paramount role during learning.

 We set position marked with $G$ in Figure 6 as the goal of task $\Omega_{s}$ to present a case when PRQL does not converge to the greedy policy.
 We show the probability of executing each policy by  PRQL  to solve  $\Omega_{s}$ in Figure 9.
 The goal of $\Omega_s$ is just on the way to the goal of $\Omega_1$, so these two tasks are especially similar.
\begin{figure}[htbp]
   \centering
   \includegraphics[width =0.47\textwidth]{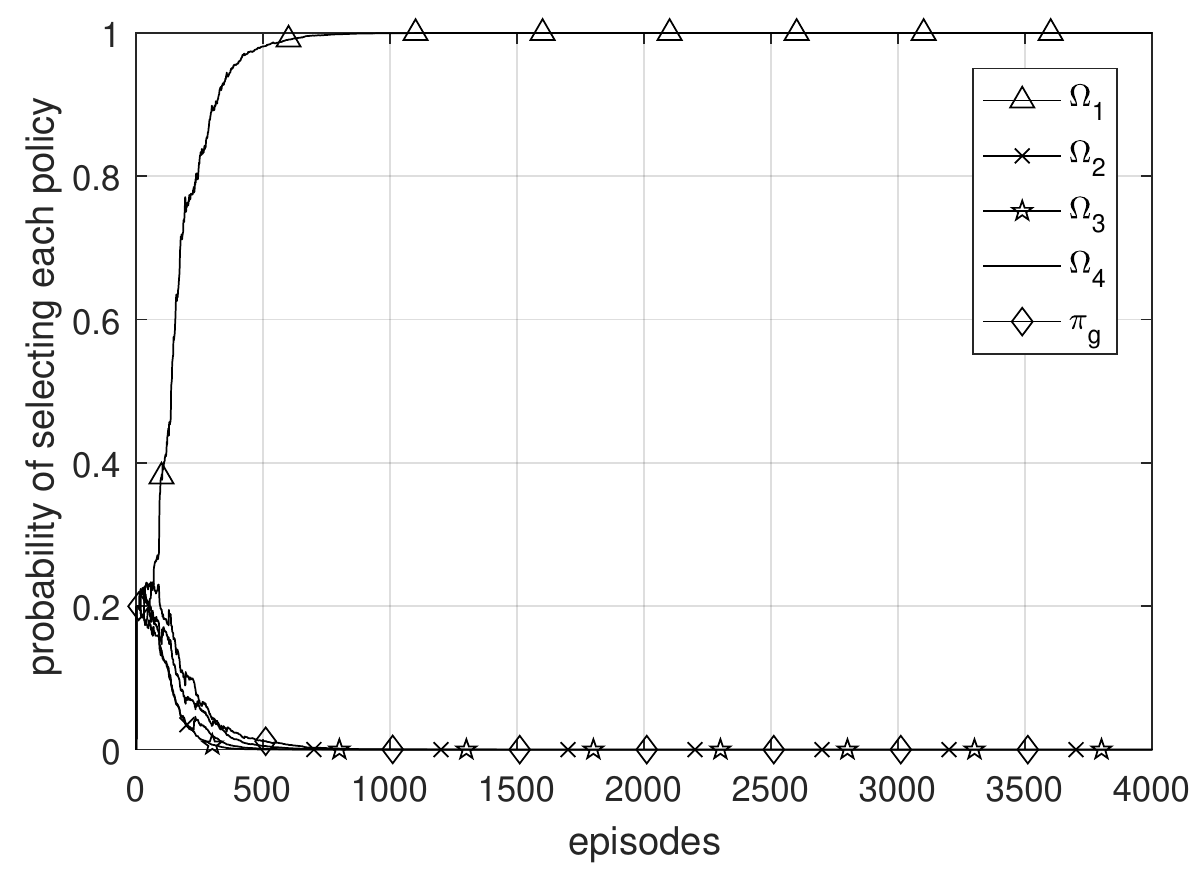}
   \setlength{\abovecaptionskip}{-0.5cm}
   \setlength{\belowcaptionskip}{-0.3cm}
   \caption{Probability of selecting each policy by PRQL to solve $\Omega_s$}
\end{figure}

PRQL ends up reusing the most similar source task $\Omega_1$ to solve $\Omega_s$ rather than the greedy policy in 2 of the 10 times. When there is a source task especially similar to the target task, the reward of reusing the most similar policy may exceed the reward of the greedy policy, so PRQL converges to the most relevant source task.
In contrast, our algorithm controls the exploration rate by tuning $p_t$.
We choose $\epsilon\text{-}greedy$ strategy with a probability of $1-p_t$.
 As $p_t$ decreases over time, our algorithm will invariably converges to $\epsilon\text{-}greedy$ policy no matter how similar the target task is to source tasks.

\section{Summary and Future Directions}
This work focuses on transfer learning in RL.
In this paper, we develop an optimal online method of selecting source policies.
Our method formulates online source policy selection as an MAB problem.
In contrast to previous works, this work provides firm theoretical ground to achieve the optimal source policy selection process.
In addition, we augment Q-learning with policy reuse and maintain the same theoretical guarantee of convergence as tradional Q-learning.
Furthermore, we present empirical validation that our algorithm outperforms the state-of-the-art transfer learning method and promotes transfer successfully in practice.

These promising results suggest several interesting directions for future research.
One of them is to combine inter-mapping between source tasks and target tasks with policy reuse. So we are able to deal with more general circumstance of different state and action spaces.
Second, we intend to formulate source task selection problem in an MDP setting and select source tasks based on the current state, thus the moment of transfer is determined automatically. Finally, it's of importance to extend the proposed algorithm to deep RL and test it in benchmark problems.

\bibliographystyle{aaai} \bibliography{1,2,3,4,5,6,7,8,9,10,11,12,13,14,15,16,17,18,19,20,21,22,23,24,25,26,27,28,29,30,31,32,37,38,39,40,41,42,43,44,45,48,49,50}

\begin{thebibliography}{}

\bibitem[\protect\citeauthoryear{Ammar and
  Taylor}{2011}]{ammar2011reinforcement}
Ammar, H.~B., and Taylor, M.~E.
\newblock 2011.
\newblock Reinforcement learning transfer via common subspaces.
\newblock In {\em International Workshop on Adaptive and Learning Agents},
  21--36.
\newblock Springer.

\bibitem[\protect\citeauthoryear{Ammar \bgroup et al\mbox.\egroup
  }{2014}]{ammar2014automated}
Ammar, H.~B.; Eaton, E.; Taylor, M.~E.; Mocanu, D.~C.; Driessens, K.; Weiss,
  G.; and Tuyls, K.
\newblock 2014.
\newblock An automated measure of mdp similarity for transfer in reinforcement
  learning.

\bibitem[\protect\citeauthoryear{Ammar \bgroup et al\mbox.\egroup
  }{2015}]{ammar2015unsupervised}
Ammar, H.~B.; Eaton, E.; Ruvolo, P.; and Taylor, M.~E.
\newblock 2015.
\newblock Unsupervised cross-domain transfer in policy gradient reinforcement
  learning via manifold alignment.
\newblock In {\em Proc. of AAAI}.

\bibitem[\protect\citeauthoryear{Audibert, Munos, and
  Szepesv{\'a}ri}{2009}]{audibert2009exploration}
Audibert, J.-Y.; Munos, R.; and Szepesv{\'a}ri, C.
\newblock 2009.
\newblock Exploration--exploitation tradeoff using variance estimates in
  multi-armed bandits.
\newblock {\em Theoretical Computer Science} 410(19):1876--1902.

\bibitem[\protect\citeauthoryear{Auer, Cesa-Bianchi, and
  Fischer}{2002}]{auer2002finite}
Auer, P.; Cesa-Bianchi, N.; and Fischer, P.
\newblock 2002.
\newblock Finite-time analysis of the multiarmed bandit problem.
\newblock {\em Machine learning} 47(2-3):235--256.

\bibitem[\protect\citeauthoryear{Barreto \bgroup et al\mbox.\egroup
  }{2016}]{barreto2016successor}
Barreto, A.; Munos, R.; Schaul, T.; and Silver, D.
\newblock 2016.
\newblock Successor features for transfer in reinforcement learning.
\newblock {\em arXiv preprint arXiv:1606.05312}.

\bibitem[\protect\citeauthoryear{Brunskill and Li}{2013}]{brunskill2013sample}
Brunskill, E., and Li, L.
\newblock 2013.
\newblock Sample complexity of multi-task reinforcement learning.
\newblock {\em arXiv preprint arXiv:1309.6821}.

\bibitem[\protect\citeauthoryear{Fachantidis \bgroup et al\mbox.\egroup
  }{2015}]{fachantidis2015transfer}
Fachantidis, A.; Partalas, I.; Taylor, M.~E.; and Vlahavas, I.
\newblock 2015.
\newblock Transfer learning with probabilistic mapping selection.
\newblock {\em Adaptive Behavior} 23(1):3--19.

\bibitem[\protect\citeauthoryear{Fern{\'a}ndez and
  Veloso}{2013}]{fernandez2013learning}
Fern{\'a}ndez, F., and Veloso, M.
\newblock 2013.
\newblock Learning domain structure through probabilistic policy reuse in
  reinforcement learning.
\newblock {\em Progress in Artificial Intelligence} 2(1):13--27.

\bibitem[\protect\citeauthoryear{Gupta \bgroup et al\mbox.\egroup
  }{2017}]{gupta2017learning}
Gupta, A.; Devin, C.; Liu, Y.; Abbeel, P.; and Levine, S.
\newblock 2017.
\newblock Learning invariant feature spaces to transfer skills with
  reinforcement learning.
\newblock {\em arXiv preprint arXiv:1703.02949}.

\bibitem[\protect\citeauthoryear{Lai and Robbins}{1985}]{lai1985asymptotically}
Lai, T.~L., and Robbins, H.
\newblock 1985.
\newblock Asymptotically efficient adaptive allocation rules.
\newblock {\em Advances in Applied Mathematics} 6(1):4--22.

\bibitem[\protect\citeauthoryear{Laroche and
  Barlier}{2017}]{laroche2017transfer}
Laroche, R., and Barlier, M.
\newblock 2017.
\newblock Transfer reinforcement learning with shared dynamics.
\newblock In {\em AAAI},  2147--2153.

\bibitem[\protect\citeauthoryear{Lazaric, Restelli, and
  Bonarini}{2008}]{lazaric2008transfer}
Lazaric, A.; Restelli, M.; and Bonarini, A.
\newblock 2008.
\newblock Transfer of samples in batch reinforcement learning.
\newblock In {\em Proceedings of the 25th international conference on Machine
  learning},  544--551.
\newblock ACM.

\bibitem[\protect\citeauthoryear{Lehnert, Tellex, and
  Littman}{2017}]{lehnert2017advantages}
Lehnert, L.; Tellex, S.; and Littman, M.~L.
\newblock 2017.
\newblock Advantages and limitations of using successor features for transfer
  in reinforcement learning.
\newblock {\em arXiv preprint arXiv:1708.00102}.

\bibitem[\protect\citeauthoryear{Melo}{2001}]{melo2001convergence}
Melo, F.~S.
\newblock 2001.
\newblock Convergence of q-learning: A simple proof.
\newblock {\em Institute Of Systems and Robotics, Tech. Rep}  1--4.

\bibitem[\protect\citeauthoryear{Nguyen, Silander, and
  Leong}{2012}]{nguyen2012transferring}
Nguyen, T.; Silander, T.; and Leong, T.~Y.
\newblock 2012.
\newblock Transferring expectations in model-based reinforcement learning.
\newblock In {\em Advances in Neural Information Processing Systems},
  2555--2563.

\bibitem[\protect\citeauthoryear{Pan and Yang}{2010}]{pan2010survey}
Pan, S.~J., and Yang, Q.
\newblock 2010.
\newblock A survey on transfer learning.
\newblock {\em IEEE Transactions on knowledge and data engineering}
  22(10):1345--1359.

\bibitem[\protect\citeauthoryear{Parisotto, Ba, and
  Salakhutdinov}{2015}]{parisotto2015actor}
Parisotto, E.; Ba, J.~L.; and Salakhutdinov, R.
\newblock 2015.
\newblock Actor-mimic: Deep multitask and transfer reinforcement learning.
\newblock {\em arXiv preprint arXiv:1511.06342}.

\bibitem[\protect\citeauthoryear{Perkins, Precup, and
  others}{1999}]{perkins1999using}
Perkins, T.~J.; Precup, D.; et~al.
\newblock 1999.
\newblock Using options for knowledge transfer in reinforcement learning.
\newblock {\em University of Massachusetts, Amherst, MA, USA, Tech. Rep}.

\bibitem[\protect\citeauthoryear{Ramakrishnan, Zhang, and
  Shah}{2017}]{ramakrishnan2017perturbation}
Ramakrishnan, R.; Zhang, C.; and Shah, J.
\newblock 2017.
\newblock Perturbation training for human-robot teams.
\newblock {\em Journal of Artificial Intelligence Research} 59:495--541.

\bibitem[\protect\citeauthoryear{Rosman, Hawasly, and
  Ramamoorthy}{2016}]{rosman2016bayesian}
Rosman, B.; Hawasly, M.; and Ramamoorthy, S.
\newblock 2016.
\newblock Bayesian policy reuse.
\newblock {\em Machine Learning} 104(1):99--127.

\bibitem[\protect\citeauthoryear{Rusu \bgroup et al\mbox.\egroup
  }{2016}]{rusu2016progressive}
Rusu, A.~A.; Rabinowitz, N.~C.; Desjardins, G.; Soyer, H.; Kirkpatrick, J.;
  Kavukcuoglu, K.; Pascanu, R.; and Hadsell, R.
\newblock 2016.
\newblock Progressive neural networks.
\newblock {\em arXiv preprint arXiv:1606.04671}.

\bibitem[\protect\citeauthoryear{Sinapov \bgroup et al\mbox.\egroup
  }{2015}]{sinapov2015learning}
Sinapov, J.; Narvekar, S.; Leonetti, M.; and Stone, P.
\newblock 2015.
\newblock Learning inter-task transferability in the absence of target task
  samples.
\newblock In {\em Proceedings of the 2015 International Conference on
  Autonomous Agents and Multiagent Systems},  725--733.
\newblock International Foundation for Autonomous Agents and Multiagent
  Systems.

\bibitem[\protect\citeauthoryear{Song \bgroup et al\mbox.\egroup
  }{2016}]{song2016measuring}
Song, J.; Gao, Y.; Wang, H.; and An, B.
\newblock 2016.
\newblock Measuring the distance between finite markov decision processes.
\newblock In {\em Proceedings of the 2016 International Conference on
  Autonomous Agents \& Multiagent Systems},  468--476.
\newblock International Foundation for Autonomous Agents and Multiagent
  Systems.

\bibitem[\protect\citeauthoryear{Sutton and
  Barto}{1998}]{sutton1998reinforcement}
Sutton, R.~S., and Barto, A.~G.
\newblock 1998.
\newblock {\em Reinforcement learning: An introduction}, volume~1.
\newblock MIT press Cambridge.

\bibitem[\protect\citeauthoryear{Talvitie and
  Singh}{2007}]{talvitie2007experts}
Talvitie, E., and Singh, S.~P.
\newblock 2007.
\newblock An experts algorithm for transfer learning.
\newblock In {\em IJCAI},  1065--1070.

\bibitem[\protect\citeauthoryear{Taylor and Stone}{2009}]{taylor2009transfer}
Taylor, M.~E., and Stone, P.
\newblock 2009.
\newblock Transfer learning for reinforcement learning domains: A survey.
\newblock {\em Journal of Machine Learning Research} 10(Jul):1633--1685.

\bibitem[\protect\citeauthoryear{Taylor, Stone, and
  Liu}{2007}]{taylor2007transfer}
Taylor, M.~E.; Stone, P.; and Liu, Y.
\newblock 2007.
\newblock Transfer learning via inter-task mappings for temporal difference
  learning.
\newblock {\em Journal of Machine Learning Research} 8(Sep):2125--2167.

\bibitem[\protect\citeauthoryear{Torrey \bgroup et al\mbox.\egroup
  }{2005}]{torrey2005using}
Torrey, L.; Walker, T.; Shavlik, J.; and Maclin, R.
\newblock 2005.
\newblock Using advice to transfer knowledge acquired in one reinforcement
  learning task to another.
\newblock In {\em ECML},  412--424.
\newblock Springer.

\bibitem[\protect\citeauthoryear{Wilson \bgroup et al\mbox.\egroup
  }{2007}]{wilson2007multi}
Wilson, A.; Fern, A.; Ray, S.; and Tadepalli, P.
\newblock 2007.
\newblock Multi-task reinforcement learning: a hierarchical bayesian approach.
\newblock In {\em Proceedings of the 24th international conference on Machine
  learning},  1015--1022.
\newblock ACM.

\end{thebibliography}

\end{document}